\documentclass{article} 
\usepackage{iclr2027_conference,times}

\usepackage{amsmath,amsfonts,bm}

\def\eqref#1{equation~\ref{#1}}

\def\1{\bm{1}}

\DeclareMathAlphabet{\mathsfit}{\encodingdefault}{\sfdefault}{m}{sl}
\SetMathAlphabet{\mathsfit}{bold}{\encodingdefault}{\sfdefault}{bx}{n}

\usepackage{hyperref}
\usepackage{url}

\usepackage{amssymb} 
\usepackage{graphicx}
\usepackage{float}
\usepackage{amsthm}
\newtheorem{proposition}{Proposition}

\usepackage{multirow}
\theoremstyle{remark}

\usepackage[table]{xcolor}
\usepackage{colortbl}
\usepackage{wrapfig}
\usepackage{subcaption}
\usepackage{booktabs}
\usepackage{wrapfig}
\usepackage{fontawesome5}

\title{Flash-WAM: Modality-Aware Distillation for World
Action Models}

\iclrfinalcopy
\author{Arman Akbari$^{1\dagger}$ \quad
Ci Zhang $^{2}$ \quad
Arash Akbari$^{1}$ \quad
Lin Zhao$^{1}$ \quad
Yixiao Chen$^{1}$ \\
\bf Weiwei Chen$^{3}$ \quad
Xuan Zhang$^{1}$ \quad 
Geng Yuan$^{2}$ \quad 
Yanzhi Wang$^{1,3}$ \\
$^{1}$Northeastern University \quad $^{2}$University of Georgia \quad $^{3}$EmbodyX
\\
\\
\multicolumn{1}{@{}c@{}}{\makebox[\textwidth]{\faGlobe\, \textbf{Project Page:}\, \href{https://flashwam.github.io}{\texttt{flashwam.github.io}}}}
}

\begin{document}

\maketitle

\lhead{Preprint}
{\renewcommand{\thefootnote}{\fnsymbol{footnote}}
\footnotetext[2]{Corresponding author: \href{mailto:akbari.ara@northeastern.edu}{\texttt{akbari.ara@northeastern.edu}}.}}

\begin{abstract}
World-action models (WAMs) jointly generate future video and robot actions through iterative diffusion, achieving strong performance on manipulation benchmarks but requiring tens of sequential denoising steps per chunk, a cost that precludes real-time control. 
Step distillation has emerged as the natural remedy, but off-the-shelf methods break down in the joint video-action setting. Video latents are high-dimensional and redundant, so they tolerate small errors, whereas actions are low-dimensional and precision-critical. Existing distillation methods are designed for a single modality and distill both streams in the same way, which causes task success to collapse. We introduce \textbf{Flash-WAM}, a modality-aware step-distillation framework inspired by consistency distillation. We show that the standard consistency function provides almost no learning signal at low noise levels, where actions receive their final refinement. Flash-WAM therefore derives a separate consistency function for each stream. The action stream uses a function whose learning signal scales linearly with the noise level, and video keeps the standard function. We apply Flash-WAM to two WAMs with different architectures, LingBot-VA and Motus, reducing inference to a single denoising step per modality. On LingBot-VA, this cuts per-chunk latency from $8.1$\,s to $347$\,ms on an NVIDIA L40S ($\textbf{23}\times$) while retaining $81.4\%$ success on RoboTwin 2.0, versus $36.3\%$ for naive consistency distillation, and $95.1\%$ on LIBERO. In real-world experiments on a Unitree G1 humanoid robot, Flash-WAM outperforms video-only distillation by $16.7$ points on Lingbot-VA. On Motus, a single step retains 77.0\% success, compared with  23.0\% for naive distillation. 
\end{abstract}

\section{Introduction}
\label{introduction}

World-action models (WAMs)~\citep{lingbotva,motus,dreamzero} learn robot policies by coupling future visual prediction with action generation. Generating both modalities requires many iterative denoising steps, which makes WAMs expensive to run and dominates their per-chunk latency. For example, LingBot-VA~\citep{lingbotva} needs 75 denoising steps and 8.1\,s to generate a single chunk in simulation on an NVIDIA L40S GPU (Figure~\ref{fig:teaser}(a)), far beyond the latency budget of real-time control. Prior work improves WAM efficiency by skipping video generation at test time~\citep{fastwam,gigaworld}, predicting the future in a compact latent space~\citep{lawam}, using smaller video models with sparser tokens~\citep{efficientwam}, or caching computation across chunks~\citep{c3ache}. Yet for WAMs that explicitly predict the future, the number of sequential denoising steps remains the main source of latency.

\begin{figure}[t]
    \centering
    \begin{minipage}[t]{0.62\textwidth}
        \centering
        \includegraphics[width=\textwidth]{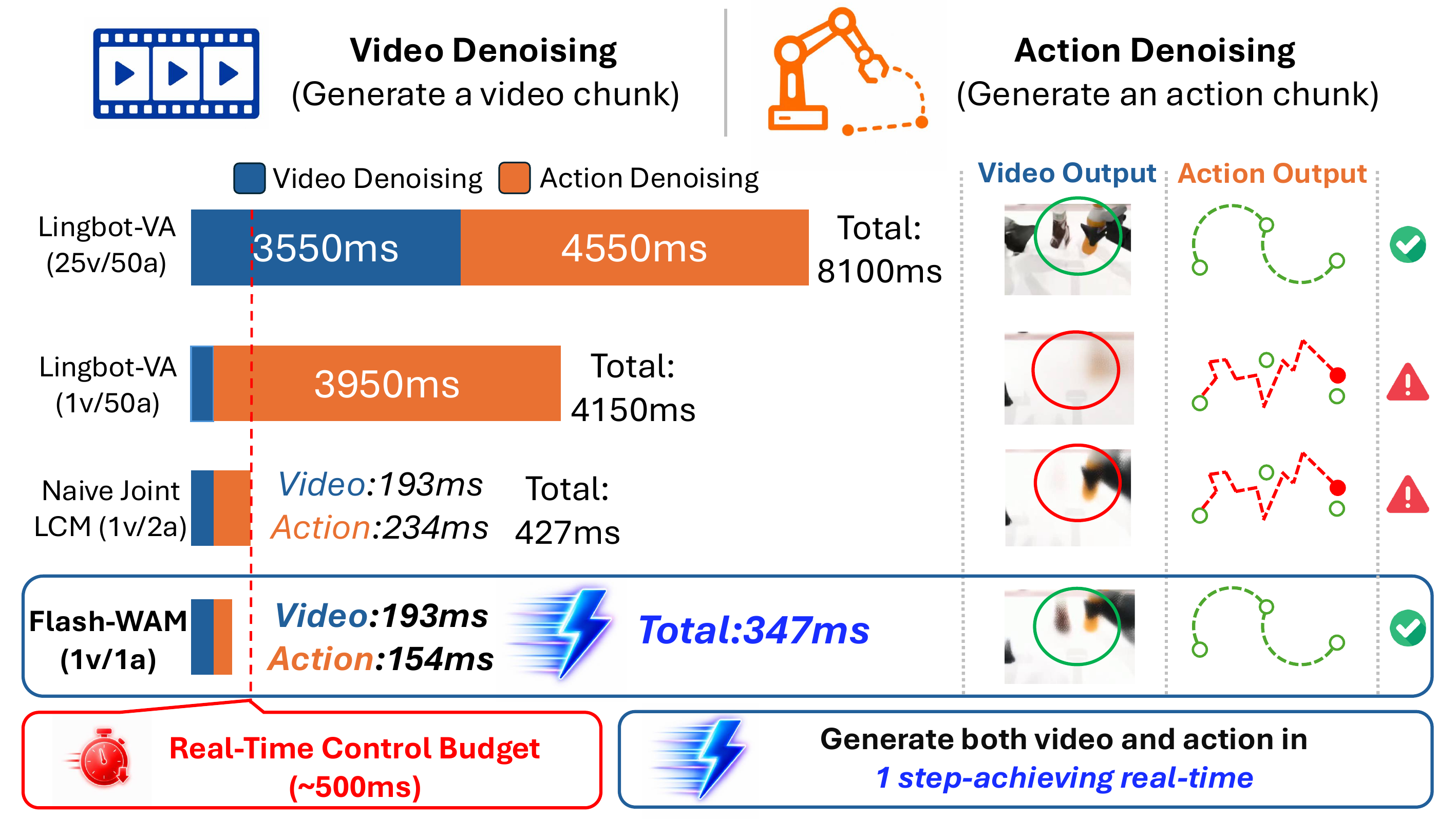}
        \subcaption{Per-chunk inference latency.}
        \label{fig:teaser_left}
    \end{minipage}\hfill
    \begin{minipage}[t]{0.35\textwidth}
        \centering
        \includegraphics[width=\textwidth]{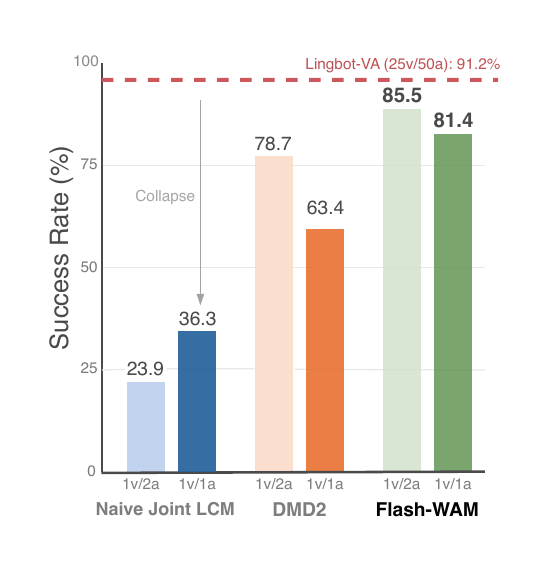}
        \subcaption{RoboTwin 2.0 success rate.}
        \label{fig:barchart}
    \end{minipage}
    \caption{\textbf{(a)} Per-chunk inference latency on a single NVIDIA L40S. \textbf{(b)} Average success rate on RoboTwin 2.0. Off-the-shelf distillation methods drop sharply, while Flash-WAM preserves \emph{teacher}-level performance.}
    \label{fig:teaser}
\end{figure}

A natural remedy is \emph{step distillation}, which enables diffusion models to generate comparable results with far fewer denoising steps. Step distillation has been extensively developed for single-modality generation and has demonstrated substantial speedups for image and video synthesis~\citep{videolcm,dmd2}. However, transferring these methods to the joint video-action setting is non-trivial. Distribution-matching approaches such as DMD2~\citep{dmd2} train auxiliary score networks alongside the student, which must be adapted to the separate noise schedule of each modality. Progressive distillation~\citep{progressive} requires a multi-stage training pipeline that scales poorly to the large pretrained backbones of modern WAMs. Consistency distillation~\citep{cd,lcm} is the most natural fit, since it requires no auxiliary networks and integrates cleanly into existing flow-matching frameworks. Yet applying it naively to a joint video-action model reduces the average success rate of LingBot-VA on RoboTwin from $91.3\%$ to $24.0\%$ (Figure~\ref{fig:teaser}(b) and Section~\ref{sec:ablation}).

This collapse stems from a mismatch between standard consistency distillation and the two modalities of a WAM. Video latents are high-dimensional and structurally redundant, whereas actions are low-dimensional and precision-critical, so WAMs assign each modality its own noise schedule with a different signal-to-noise ratio (SNR) shift~\citep{lingbotva,dreamzero}. Under these schedules, video noise concentrates at high $\sigma$, while action noise spreads across the full range with substantial mass at low $\sigma$. Standard consistency distillation, however, is developed for single-modality generation and implicitly assumes that the entire output shares one noise regime. We show that standard LCM parametrization~\citep{lcm} provides a gradient signal that vanishes \emph{quadratically} as $\sigma \to 0$, leaving the action stream with little learning signal over much of its training distribution. Naive joint distillation therefore leaves the action stream poorly distilled, which drives the collapse in task performance.

After identifying this incompatibility, we propose \textbf{Flash-WAM}, a step-distillation framework for  video-action diffusion models. The core idea is to treat video and action distillation as fundamentally different problems with different learning-signal requirements. For the action stream, we derive the consistency function whose learning signal scales linearly with the noise level. For the video stream, we keep the standard function, which already distills video well. We apply Flash-WAM to two WAMs with different architectures, Motus~\citep{motus} and LingBot-VA~\citep{lingbotva}, the state-of-the-art open-source joint video-action diffusion model on manipulation benchmarks. On Lingbot-VA, Flash-WAM recovers task success with one video and one action step, achieving a $23{\times}$ speedup that brings per-chunk inference latency from $8.1$s to $347$ms on a single NVIDIA L40S, enabling real-time inference. Our contributions are as follows:
\begin{itemize}
    \item \textbf{Diagnosis of joint-modality distillation failure.} We identify and characterize the structural failure mode that prevents off-the-shelf consistency distillation methods from succeeding in the joint video-action regime, with formal analysis and empirical experiments.

    \item \textbf{Modality-aware consistency distillation.} We propose Flash-WAM, which derives a separate consistency function for each modality based on its noise regime. The framework is grounded in a structural 
    analysis of the consistency-function family, characterizing the achievable gradient scaling under the consistency boundary condition.

\item \textbf{Real-time WAM inference.} Flash-WAM compresses LingBot-VA from $25$ video and $50$ action denoising steps to one video step and two action steps, a $19{\times}$ speedup, while retaining $85.5\%$ success on RoboTwin 2.0 and $96.1\%$ on LIBERO, and achieving $60\%$ on a Unitree G1 humanoid robot versus $43.3\%$ for video-only LCM.
On Motus, a single step retains $77.0\%$ success compared with $82.8\%$ for the 10-step teacher.

\end{itemize}
\section{Related Works}
\label{related}

\paragraph{World Action Models.}
WAMs couple video and action generation in a single framework. LingBot-VA~\citep{lingbotva} casts policy generation as autoregressive video-action diffusion, Motus~\citep{motus} couples a vision-language model, video generator, and action generator in a Mixture-of-Transformers architecture, and DreamZero~\citep{dreamzero} reduces denoising steps through architecture-level inference optimizations. All share an inference bottleneck dominated by iterative video and action denoising. A complementary line of work avoids test-time video generation: GigaWorld-Policy~\citep{gigaworldpolicy} uses future visual dynamics as a reasoning signal under a causal mask rather than an explicit prediction, and Fast-WAM~\citep{fastwam} repurposes a pretrained video DiT as a single-pass encoder for action generation. Rather than removing video generation, Flash-WAM accelerates it through step distillation, preserving the original WAM inference structure while reducing each modality's denoising to as few as one step.

\paragraph{Step Distillation.}
Step distillation compresses the iterative denoising of diffusion models into a few inference steps. \emph{Trajectory-following} methods~\citep{progressive,cd,lcm,easycd,shortcut} train the student to follow the teacher's ODE trajectory, for example by iteratively halving the number of steps~\citep{progressive} or by mapping every point on the trajectory to its clean endpoint~\citep{cd,lcm,easycd}. \emph{Distribution-matching} methods~\citep{dmd2,sim,fdiverg,phaseddmd} instead match the student's output distribution to the teacher's using auxiliary score networks. Both families have been extended to video diffusion~\citep{videolcm,animatelcm,dollar,tdm}, but they are designed for single-modality generation under a single noise distribution. We build on trajectory-following methods, which integrate cleanly into flow-matching frameworks and admit the gradient analysis of Section~\ref{sec:naive_fails}.

\section{Preliminaries}
\label{preliminaries}

\subsection{Background}
\label{sec:background}
\paragraph{Flow Matching.}
Flow matching~\citep{fm,rectified} is a continuous-time generative framework that learns to transport samples from a noise distribution to a data distribution along straight-line interpolation paths. Clean data $\mathbf{x}_0$ is corrupted to $\mathbf{x}_\sigma = (1-\sigma)\,\mathbf{x}_0 + \sigma\,\boldsymbol{\epsilon}$ with $\boldsymbol{\epsilon} \sim \mathcal{N}(\mathbf{0},\mathbf{I})$ and $\sigma \in [0,1]$. A neural network $v_\theta$ is trained to predict the velocity $v = \boldsymbol{\epsilon} - \mathbf{x}_0 = d\mathbf{x}_\sigma/d\sigma$ via the flow matching objective:
\begin{equation}
\mathcal{L}_{\text{FM}} = \mathbb{E}_{\mathbf{x}_0, \boldsymbol{\epsilon}, \sigma} \left\| v_\theta(\mathbf{x}_\sigma, \sigma) - (\boldsymbol{\epsilon} - \mathbf{x}_0) \right\|^2,
\label{eq:fm_loss}
\end{equation}
from which the clean estimate is recovered as $\hat{\mathbf{x}}_0 = \mathbf{x}_\sigma - \sigma\,v_\theta$. To control where training mass concentrates along the noise schedule, an SNR-shifted sampler is commonly used:
\begin{equation}
\sigma = \frac{s\,\tilde{\sigma}}{1 + (s-1)\,\tilde{\sigma}}, \qquad \tilde{\sigma} \sim \mathcal{U}[0,1],
\label{eq:snr_shift}
\end{equation}
parametrized by a shift $s \geq 1$; larger $s$ pushes the distribution toward higher noise levels. Samples are generated at inference by numerical Euler integration of the velocity field from $\sigma = 1$ to $\sigma = 0$.

\paragraph{Consistency Distillation.}
Consistency models~\citep{cd,lcm,easycd} accelerate sampling by enforcing the \emph{consistency property}: a consistency function $f(\mathbf{x}_\sigma, \sigma)$ maps any point on the probability flow Ordinary Differential Equation (ODE) trajectory to its clean endpoint at $\sigma = 0$. The general form is
\begin{equation}
f(\mathbf{x}_\sigma, \sigma) = a(\sigma)\,\mathbf{x}_\sigma + b(\sigma)\,v_\theta(\mathbf{x}_\sigma, \sigma),
\label{eq:consistency_general}
\end{equation}
where $a, b: [0,1] \to \mathbb{R}$ satisfy the boundary condition $a(0)=1,\; b(0)=0$ that enforces $f(\mathbf{x}_0, 0) = \mathbf{x}_0$. Standard parametrization~\citep{edm,cd} takes $a(\sigma) = c_{\text{skip}}(\sigma) + c_{\text{out}}(\sigma)$ and $b(\sigma) = -c_{\text{out}}(\sigma)\,\sigma$, with $c_{\text{skip}} = \sigma_d^2/(\sigma^2 + \sigma_d^2)$ and $c_{\text{out}} = \sigma\sigma_d/\sqrt{\sigma^2 + \sigma_d^2}$.

In the distillation setting, a frozen teacher $\theta_T$ provides guided Euler steps while a student $\theta_S$ and an Exponential Moving Average (EMA) target $\theta_{S'}$ are trained to agree along these trajectories. At each iteration, a noise level $\sigma_s$ is sampled, the schedule is advanced $k$ discrete steps to obtain $\sigma_e < \sigma_s$, and the target $\tilde{\mathbf{x}}_{\sigma_e} = \mathbf{x}_{\sigma_s} + \hat{v}_{\text{cfg}}\,(\sigma_e - \sigma_s)$ is formed via a teacher Euler step (with classifier-free guidance during distillation). The student is trained against this target via the consistency loss
$\mathcal{L}_{\text{CD}} = d\!\left(f_{\theta_S}(\mathbf{x}_{\sigma_s}, \sigma_s),\; f_{\theta_{S'}}(\tilde{\mathbf{x}}_{\sigma_e}, \sigma_e)\right)$,
where $d$ is a distance metric.

\subsection{Problem Formulation}
\label{problem_formulation}

World-action models (WAMs) decompose policy generation into two coupled stages: \emph{visual dynamics prediction}, which predicts how the world will evolve in latent space, and \emph{inverse dynamics}, which recovers the actions consistent with that predicted transition. Given a context $\mathbf{C}$ summarizing past observations, past actions, and a language instruction, a WAM jointly samples a chunk of $K$ future video latents $\mathbf{x}^v$ and the corresponding action sequence $\mathbf{x}^a$:
\begin{equation}
\mathbf{x}^v \sim p_\theta\!\left(\mathbf{x}^v \mid \mathbf{C}\right) \;\; \text{(visual dynamics)},
\qquad
\mathbf{x}^a \sim p_\theta\!\left(\mathbf{x}^a \mid \mathbf{x}^v,\, \mathbf{C}\right) \;\; \text{(inverse dynamics)}.
\label{eq:wam}
\end{equation}
This factorization grounds actions in predicted future states. Visual dynamics is sampled by Euler integration of a video velocity field $v^v_\theta$ over $N^v$ steps, and inverse dynamics by integration of an action velocity field $v^a_\theta$ over $N^a$ steps. Generating one chunk therefore requires $N^v + N^a$ sequential transformer forward passes, dominating per-chunk latency and preventing real-time control. We use the shorthand $N^v$v/$N^a$a to denote a specific Number of Function Evaluation (NFE) configuration; for example, $25$v/$50$a denotes $25$ video and $50$ action denoising steps.

\begin{figure}[t]
    \centering
    \includegraphics[width=\linewidth]{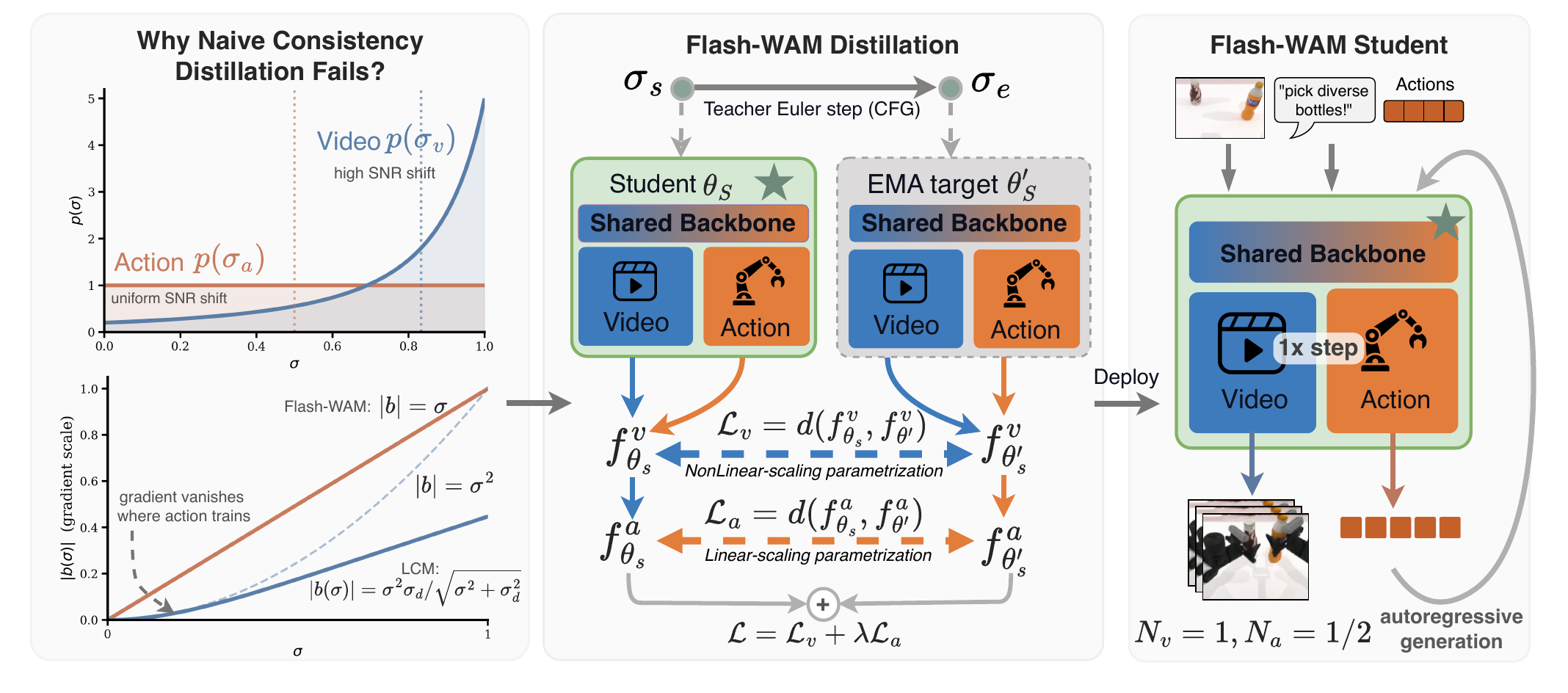}
    \vspace{-16pt} 
    \caption{Overview of Flash-WAM. \textbf{Left:} the standard consistency function passes almost no learning signal at low noise levels, where actions are refined. \textbf{Middle:} the Flash-WAM training pipeline with modality-aware consistency functions. \textbf{Right:} the distilled \emph{student} at deployment with single denoising step.}
    \label{fig:overview}
\end{figure}

Although both stages are realized as flow matching, they are trained under different noise schedules. Following standard practice~\citep{lingbotva, dreamzero}, the two stages use independent SNR-shifted schedulers (Eq.~\ref{eq:snr_shift}) with per-modality shift parameters $s^v$ (video) and $s^a$ (action) satisfying
$s^v > s^a$,
reflecting that high-dimensional, structurally redundant video latents remain informative at higher noise levels, while low-dimensional, precision-critical action sequences require a gentler schedule. As a result, video noise concentrates at high $\sigma$, while action noise spreads across the full range with substantial mass at low $\sigma$. Existing step-distillation methods are designed for a single modality with one noise distribution, and therefore do not account for this asymmetry when reducing $N^v$ and $N^a$.

\section{Methodology}
\label{methodology}

We introduce \textbf{Flash-WAM}, a step-distillation framework for joint video-action models. Section~\ref{sec:naive_fails} analyzes how a consistency function distributes learning signal across noise levels and shows why a single function cannot serve both streams. Section~\ref{sec:adaptive} uses this analysis to choose a consistency function for each stream and combines the two in a joint objective. The framework changes only the training loss, so it applies to any WAM whose video and action streams are trained with flow matching, regardless of how their parameters are organized.

\subsection{The Joint Distillation Regime}
\label{sec:naive_fails}

The most direct approach to distilling a joint video-action model applies a single consistency function uniformly to both modalities, ignoring that the two modalities differ in how much a small error costs. This approach fails because of how the consistency loss distributes learning signal across noise levels. Any consistency function of the form in Eq.~\ref{eq:consistency_general} depends on $\theta$ only through $v_\theta$, so its gradient with respect to $\theta$ is scaled by $|b(\sigma)|$, and so is the gradient of the consistency loss. Whenever $|b(\sigma)|$ is small, the network receives little learning signal at noise level $\sigma$, regardless of the prediction quality of $v_\theta$. The choice of $b$ therefore determines where in the noise schedule the model can effectively learn, and it must provide learning signal where each modality's accuracy is decided.

As a concrete example, consider the LCM parametrization, for which $b_{\mathrm{LCM}}(\sigma) = -\sigma^2 \sigma_d / \sqrt{\sigma^2 + \sigma_d^2}$. Since $b_{\mathrm{LCM}}(0) = b_{\mathrm{LCM}}'(0) = 0$, a Taylor expansion gives $|b_{\mathrm{LCM}}(\sigma)| = \sigma^2 + \mathcal{O}(\sigma^4)$, which vanishes quadratically as $\sigma \to 0$. The left panel of Figure~\ref{fig:overview} illustrates the gap: at $\sigma = 0.1$, $|b_{\mathrm{LCM}}(\sigma)|$ (blue line) is roughly $10\times$ smaller than the linear value $|b(\sigma)|=\sigma$. Whether a consistency function suffers from this vanishing depends only on its behavior at $\sigma = 0$:

\begin{proposition}[Gradient scaling near $\sigma=0$]
\label{prop:optimal_order}
Let $f(\mathbf{x}_\sigma, \sigma) = a(\sigma)\mathbf{x}_\sigma + b(\sigma) v_\theta$ be any consistency function with $a(0)=1$, $b(0)=0$, and $b \in C^2([0,1])$. Then $|b(\sigma)| = |b'(0)|\,\sigma + \mathcal{O}(\sigma^2)$ as $\sigma \to 0$. In particular, $|b(\sigma)| = \Theta(\sigma)$ if and only if $b'(0) \neq 0$, and $|b(\sigma)| = \mathcal{O}(\sigma^2)$ otherwise.
\end{proposition}

\begin{proof}
By Taylor's theorem, $b(\sigma) = b'(0)\sigma + \tfrac{1}{2}b''(\xi)\sigma^2$ for some $\xi \in (0,\sigma)$. Since $b''$ is bounded on $[0,1]$, $\big|\,|b(\sigma)| - |b'(0)|\sigma\,\big| \le \tfrac{1}{2}\big(\sup_{[0,1]}|b''|\big)\sigma^2$, and both claims follow.
\end{proof}

LCM falls in the case $b'(0) = 0$ for every choice of $\sigma_d > 0$, so no member of the LCM family attains linear scaling, and the obstruction is structural rather than parametric. We therefore look outside the LCM family for the action-stream consistency function.

\subsection{Modality-Aware Consistency Functions}
\label{sec:adaptive}

\paragraph{Action stream.}
Actions are precision-critical and receive their final refinement at low $\sigma$, where Proposition~\ref{prop:optimal_order} is decisive. We therefore choose $a(\sigma) = 1$ and $b(\sigma) = -\sigma$, which satisfy the boundary condition with $b'(0) \neq 0$ and give
\begin{equation}
    f^a(\mathbf{x}^a_\sigma, \sigma) = \mathbf{x}^a_\sigma - \sigma\, v^a_\theta(\mathbf{x}^a_\sigma, \sigma).
    \label{eq:f_action}
\end{equation}
Here $|b(\sigma)| = \sigma$ is exactly linear, so no higher-order terms dampen the gradient, the constant $a$ avoids a skip term that varies with $\sigma$, and neither coefficient introduces a hyperparameter. The boundary condition $f^a(\mathbf{x}^a_0, 0) = \mathbf{x}^a_0$ holds by construction, and consistency is enforced as in standard consistency distillation~\citep{cd,lcm}.

\paragraph{Video stream.}
The action-side selection criterion (linear scaling near $\sigma = 0$) does not apply to the video stream. The video distribution concentrates at large $\sigma$, where LCM already provides ample gradient signal. Because a velocity error $\delta v$ changes the output by $b(\sigma)\,\delta v$, the standard function passes such errors to the output with a gain of at most $\sigma_d/\sqrt{1+\sigma_d^2} \approx 0.45$ for $\sigma_d = 0.5$, whereas Eq.~\ref{eq:f_action} passes them with a gain of up to $1$. Flash-WAM therefore keeps the standard function for the video stream,
\begin{equation}
    f^v(\mathbf{x}^v_\sigma, \sigma) = c_{\text{skip}}(\sigma)\, \mathbf{x}^v_\sigma + c_{\text{out}}(\sigma)\, \hat{\mathbf{x}}^v_0.
    \label{eq:f_video}
\end{equation}

\begin{table}[t]
\centering
\caption{Success rates on RoboTwin 2.0 simulation (Clean and Randomized splits, 50 tasks) and speedup over the LingBot-VA as the \emph{teacher}. ``*'' indicates results we have reproduced.}
\resizebox{0.8\textwidth}{!}{%
\begin{tabular}{lcccccc}
\toprule
Method & $N^v$ & $N^a$ & Clean & Rand. & Average & Speedup \\
\midrule
$\pi_0$~\citep{pi0}                          & -- & -- & 65.92 & 58.40 & 62.2  & -- \\
$\pi_{0.5}$~\citep{pi05}                     & -- & -- & 82.74 & 76.76 & 79.8  & -- \\
X-VLA~\citep{xvla}                           & -- & -- & 72.9  & 72.8  & 72.8  & -- \\
Motus~\citep{motus}                          & -- & -- & 88.66 & 87.02 & 87.8  & -- \\
LingBot-VA*~\citep{lingbotva}                & 25 & 50 & 91.64 & 90.86 & 91.25 & 1.0$\times$ \\
\midrule
LingBot-VA + DMD2                            & 1 & 2 & 85.08              & 72.36  & 78.74    & \multirow{4}{*}{19.0$\times$}\\
LingBot-VA + Video-only LCM                  & 1 & 2 & 80.66           & 76.92  & 78.79 & \\
LingBot-VA + Naive Joint LCM                 & 1 & 2 & 25.88           & 22.07  & 23.97 & \\
\rowcolor{gray!15}
\textbf{Ours}                                & 1 & 2 & \textbf{88.42}  & \textbf{82.66} & \textbf{85.54} & \multicolumn{1}{>{\columncolor{white}}c}{} \\
\midrule
LingBot-VA + DMD2                            & 1 & 1 & 66.53     & 60.32     & 63.43     & \multirow{4}{*}{23.3$\times$}\\
LingBot-VA + Video-only LCM                  & 1 & 1 & 77.90  & 69.46  & 73.68 & \\
LingBot-VA + Naive Joint LCM                 & 1 & 1 & 39.68     & 32.96  & 36.32    & \\
\rowcolor{gray!15}
\textbf{Ours}                                & 1 & 1 & \textbf{82.56}  & \textbf{80.26} & \textbf{81.41} & \multicolumn{1}{>{\columncolor{white}}c}{} \\
\bottomrule
\end{tabular}%
}
\label{tab:main_results}
\end{table}

\begin{table}[t]
\centering
\caption{Success rates on LIBERO benchmarks (Spatial, Object, Goal, Long-horizon) and speedup over the LingBot-VA \emph{teacher}. ``*'' indicates results we have reproduced.}
\resizebox{0.9\textwidth}{!}{%
\begin{tabular}{lccccccccc}
\toprule
Method & $N^v$ & $N^a$ & Spatial & Object & Goal & Long & Average & Speedup \\
\midrule
$\pi_0$~\citep{pi0}                       & -- & -- & 96.8 & 98.8 & 95.8 & 85.2 & 94.1 & -- \\
X-VLA~\citep{xvla}                        & -- & -- & 98.2 & 98.6 & 97.8 & 97.6 & 98.1 & -- \\
LingBot-VA*~\citep{lingbotva}             & 20 & 50 & 98.5   & 99.8   & 98.0   & 98.3   & 98.6   & 1.0$\times$ \\
\midrule
LingBot-VA + Video-only LCM               & 1 & 2 & 95.1 & 92.0 & 96.0   & 97.8   & 95.2 & 13.7$\times$ \\
\rowcolor{gray!15}
\textbf{Ours}                             & 1 & 2 & \textbf{97.0} & \textbf{92.8} & \textbf{96.4} & \textbf{98.0} & \textbf{96.1} & 13.7$\times$ \\
\midrule
LingBot-VA + Video-only LCM               & 1 & 1 & 95.0 & 91.5 & 95.0   & 95.4  & 94.2 & 16.8$\times$ \\
\rowcolor{gray!15}
\textbf{Ours}                             & 1 & 1 & \textbf{96.0} & \textbf{92.6} & \textbf{96.0} & \textbf{95.8} & \textbf{95.1} & 16.8$\times$ \\
\bottomrule
\end{tabular}%
}
\label{tab:libero_results}
\end{table}

\paragraph{Joint training objective.}
The student is trained to satisfy the consistency property in both modalities simultaneously. Each modality contributes a consistency loss using its own consistency function:
\begin{equation}
    \mathcal{L}^v = d\!\left(f^v_{\theta_S}(\mathbf{x}^v_{\sigma_s}, \sigma_s),\; f^v_{\theta_{S'}}(\tilde{\mathbf{x}}^v_{\sigma_e}, \sigma_e)\right),
    \qquad
    \mathcal{L}^a = d\!\left(f^a_{\theta_S}(\mathbf{x}^a_{\sigma_s}, \sigma_s),\; f^a_{\theta_{S'}}(\tilde{\mathbf{x}}^a_{\sigma_e}, \sigma_e)\right),
\end{equation}
where the teacher Euler step uses CFG with $w \sim \mathcal{U}[w_{\min}, w_{\max}]$ for video and the unguided prediction for action. The full Flash-WAM objective combines both:
\begin{equation}
    \mathcal{L} = \mathcal{L}^v + \lambda_a \, \mathcal{L}^a.
    \label{eq:full_loss}
\end{equation}
The per-stream consistency functions change only how each stream's output enters the loss. They require no change to the architecture or to the per-step compute, so Flash-WAM applies unchanged whether the two streams share one backbone or use separate networks, as in Motus~\citep{motus}. For LingBot-VA, both consistency targets are computed in one forward pass over the joint video-action sequence (Appendix~\ref{app:implementation}).

\begin{figure}[t]
    \centering
    \includegraphics[width=\linewidth]{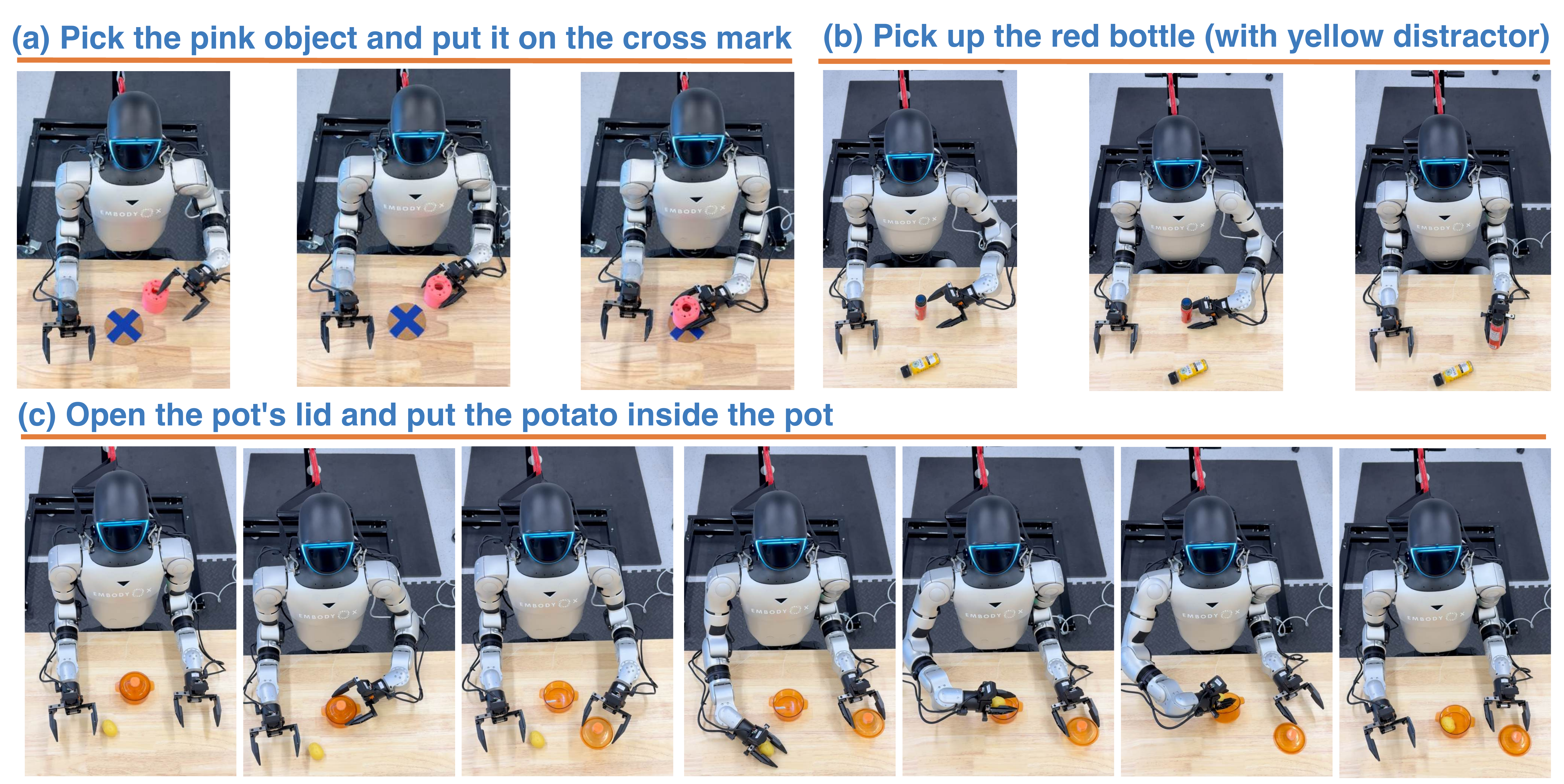}
    \vspace{-16pt} 
    \caption{Real-world tasks on a Unitree G1 humanoid robot.}
    \label{fig:realworld}
\end{figure}

\begin{table}[t]
\centering
\caption{Real-world evaluation on a Unitree G1 robot. Each entry reports successful rollouts out of 20, and the last column gives the average success rate. T1: put the potato in the pot; T2: pick the red bottle among distractors; T3: place the pink object on the cross mark.}
\label{tab:realworld}
\begin{tabular}{lccccc}
\toprule
Method & $N^v$ / $N^a$ & T1 & T2 & T3 & Average \\
\midrule
LingBot-VA (teacher)~\citep{lingbotva}        & $3$ / $10$ & 10/20 & 14/20 & 16/20 & 66.7\% \\
\midrule
LingBot-VA (reduced denoising steps)          & $1$ / $2$  & 6/20 & 6/20 & 12/20 & 40.0\% \\
LingBot-VA + Video-only LCM                   & $1$ / $2$  & 6/20 & 10/20 & 10/20 & 43.3\% \\
\rowcolor{gray!15}
\textbf{Flash-WAM}                            & $1$ / $2$  & \textbf{10/20} & \textbf{12/20} & \textbf{14/20} & \textbf{60.0\%} \\
\midrule
LingBot-VA (reduced denoising steps)          & $1$ / $1$  & 2/20 & 6/20 & 6/20 & 23.3\% \\
LingBot-VA + Video-only LCM                   & $1$ / $1$  & 4/20 & 8/20 & 8/20 & 33.3\% \\
\rowcolor{gray!15}
\textbf{Flash-WAM}                            & $1$ / $1$  & \textbf{8/20} & \textbf{10/20} & \textbf{12/20} & \textbf{50.0\%} \\
\bottomrule
\end{tabular}
\end{table}

\section{Experiments}
\label{experiment}

\subsection{Experimental Setup}
\label{sec:experimental_setup}

We apply Flash-WAM to two WAMs. LingBot-VA~\citep{lingbotva} is a state-of-the-art open-source WAM with a shared backbone, and Motus~\citep{motus} uses separate video and action networks (Table~\ref{tab:motus}). Latency is measured on a single NVIDIA L40S GPU. Because WAMs execute actions in chunks, inference is real-time when a chunk is generated before the previous one finishes executing~\citep{dreamzero,realtime}. We adopt $500$\,ms per chunk ($2$\,Hz) as the real-time budget, consistent with prior work~\citep{realtime,tidal} and below the $640$\,ms execution time of a RoboTwin 2.0 chunk.

\paragraph{Benchmarks.}
We evaluate on two simulation benchmarks and a real robot. RoboTwin 2.0~\citep{robotwin} covers 50 bimanual tasks under a Clean and a Randomized split, where the Randomized split perturbs object poses, lighting, and scene layouts. LIBERO~\citep{libero} comprises four suites (Spatial, Object, Goal, and Long-horizon) with 500 demonstrations each. For real-world evaluation, we use a Unitree G1 humanoid with Dex1-1 grippers on three tasks (Figure~\ref{fig:realworld}), (T1) opening a pot's lid and placing a potato inside, (T2) picking a red bottle next to a yellow distractor, and (T3) placing a pink object on a marked target. We collect $50$ teleoperated demonstrations per task and report success over $20$ rollouts per task and method.

\paragraph{Baselines.}
We compare against off-the-shelf step-distillation methods reimplemented for joint video-action models. \textbf{Naive joint LCM} applies the standard LCM consistency function~\citep{lcm} uniformly to both streams and is the direct counterpart to our method. \textbf{DMD2} applies distribution matching distillation~\citep{dmd2} to the video stream, with a flow-matching regularizer on the action stream. \textbf{Video-only LCM} distills only the video stream and leaves the action stream undistilled. On the real robot, we also evaluate LingBot-VA with reduced denoising steps and no distillation. We report $\pi_0$, $\pi_{0.5}$, X-VLA, and Motus as reference policies for absolute performance. Implementation details are provided in Appendix~\ref{app:implementation}.

\begin{figure}[t]
    \centering
    \includegraphics[width=\linewidth]{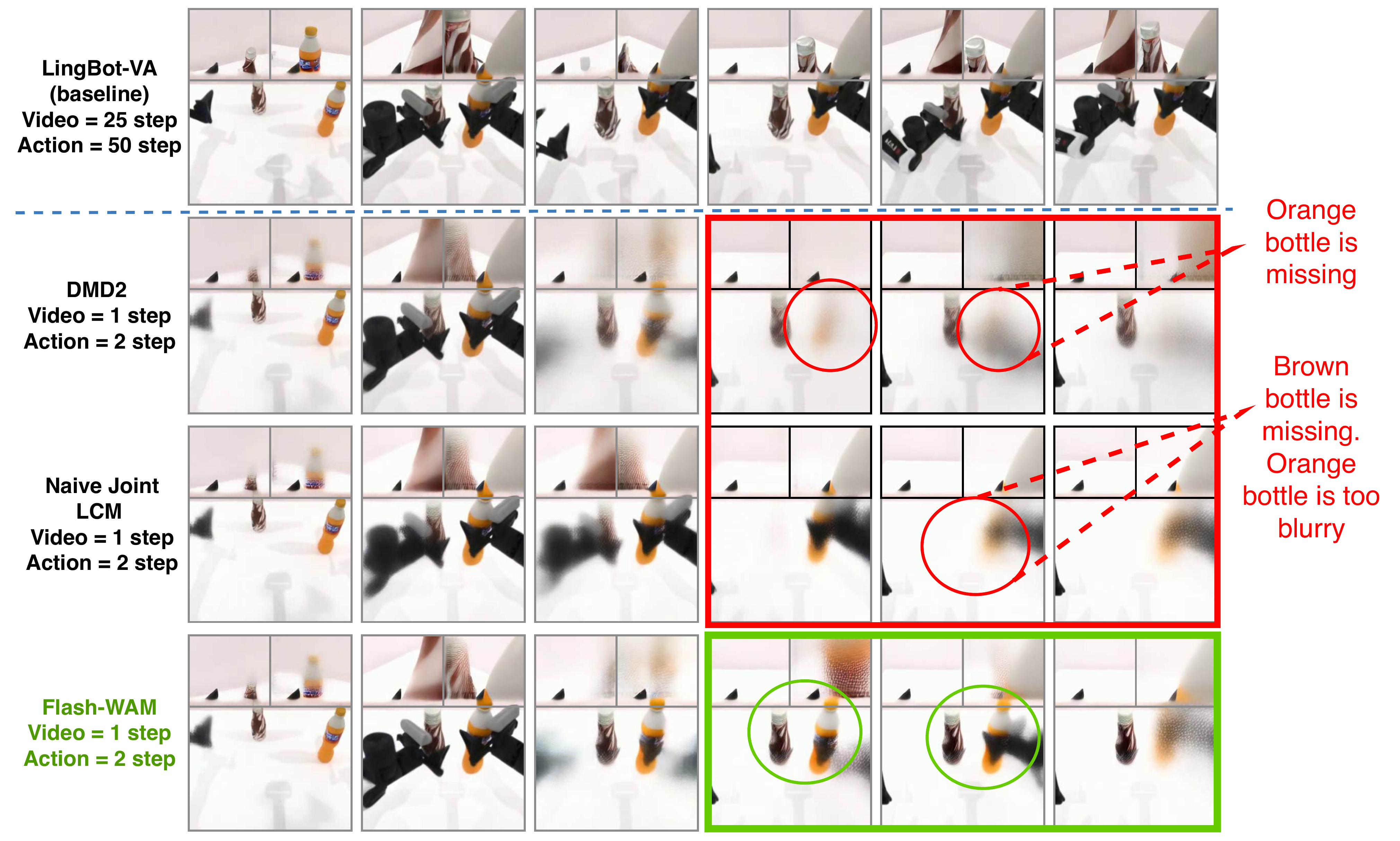}
    \vspace{-16pt} 
    \caption{Qualitative comparison on RoboTwin task ``pick diverse bottles'', generated open-loop by Flash-WAM and the distillation baselines.}
    \label{fig:qualitative}
\end{figure}

\subsection{Main Results}

\paragraph{RoboTwin.}
Table~\ref{tab:main_results} reports success rates on RoboTwin 2.0. Flash-WAM at $1$v/$2$a achieves $85.54\%$ average success, recovering most of LingBot-VA's $91.25\%$ at a $19\times$ speedup. At the more aggressive $1$v/$1$a configuration, Flash-WAM still achieves $81.41\%$ average success, within $10$ points of the unaccelerated configuration despite reducing video denoising by $25\times$ and action denoising by $50\times$. As shown in Figure~\ref{fig:teaser}(a), the corresponding $23.3\times$ speedup brings per-chunk latency down to $347$~ms on a single NVIDIA L40S, enabling real-time inference. Off-the-shelf distillation methods fall well short at the same NFE budgets. At $1$v/$2$a, naive joint LCM collapses to $23.97\%$, DMD2 reaches $78.74\%$, and video-only LCM trails at $78.79\%$. The pattern persists at $1$v/$1$a: naive joint LCM and DMD2 degrade further, while video-only LCM drops to $73.68\%$. Flash-WAM also surpasses the strongest VLA reference baselines ($\pi_0$, $\pi_{0.5}$, X-VLA) and remains competitive with Motus.

\paragraph{LIBERO.}
Table~\ref{tab:libero_results} reports success rates on the four LIBERO suites. Flash-WAM at $1$v/$2$a achieves $96.1\%$ average success, recovering nearly all of the teacher's $98.6\%$ at a $13.7\times$ speedup, and outperforms Video-only LCM on every suite. At $1$v/$1$a, Flash-WAM achieves $95.1\%$ at a $16.8\times$ speedup, reducing per-chunk latency from $6{,}767$\,ms to $404$\,ms.

\paragraph{Real-world experiments.}
Table~\ref{tab:realworld} reports real-world success rates. The released LingBot-VA model is deployed at $3$v/$10$a and achieves $66.7\%$ average success. Reducing its step budget without distillation lowers the average to $40.0\%$ at $1$v/$2$a and $23.3\%$ at $1$v/$1$a, and every task degrades. Distilling only the video stream recovers only part of this loss ($43.3\%$ and $33.3\%$), which shows that on hardware, as in simulation, the action stream must also be distilled. Flash-WAM reaches $60.0\%$ at $1$v/$2$a and $50.0\%$ at $1$v/$1$a, retaining $90\%$ and $75\%$ of the teacher's average success, and outperforms both alternatives on every task at both budgets. At $1$v/$2$a, it matches the teacher on T1 and stays within $10$ points on T2 and T3. The largest gains appear at $1$v/$1$a, where Flash-WAM improves over reduced-step inference by $30$ points on T1 ($10\% \to 40\%$) and T3 ($30\% \to 60\%$). With a single action step, Flash-WAM also outperforms both alternatives run with two action steps ($50.0\%$ versus $43.3\%$ and $40.0\%$). Even against the teacher's deployment budget of $13$ evaluations per chunk, Flash-WAM uses $4.3\times$ and $6.5\times$ fewer.

\begin{table}[t]
\centering
\caption{Ablation analysis on RoboTwin 2.0. We compare four LCM-based distillation strategies against the unaccelerated LingBot-VA at two NFE configurations, broken down by task horizon. Best per column and budget in \textbf{bold}, second \underline{underlined}.}
\label{tab:ablation_robotwin}
\resizebox{\textwidth}{!}{%
\begin{tabular}{lcccccccccc}
\toprule
& & & \multicolumn{2}{c}{Horizon = 1} & \multicolumn{2}{c}{Horizon = 2} & \multicolumn{2}{c}{Horizon = 3} & \multicolumn{2}{c}{Average (50 tasks)} \\
\cmidrule(lr){4-5}\cmidrule(lr){6-7}\cmidrule(lr){8-9}\cmidrule(lr){10-11}
Method & $N^v$ & $N^a$ & Clean & Rand. & Clean & Rand. & Clean & Rand. & Clean & Rand. \\
\midrule
LingBot-VA teacher        & 25 & 50 & 94.18 & 93.56 & 90.35 & 86.95 & 93.22 & 93.28 & 92.93 & 91.55 \\
\midrule
Video-only LCM            & 1 & 2 & 87.10 & 82.73 & 73.13 & 68.19 & 62.50 & \textbf{68.25} & 80.66 & 76.92 \\
Video-only LCM + reg.     & 1 & 2 & \underline{91.53} & \textbf{88.50} & \underline{83.00} & \underline{74.69} & \underline{68.00} & 62.75 & \underline{86.92} & \underline{82.02} \\
Naive Joint LCM           & 1 & 2 & 41.00 & 35.13 & 4.00  & 3.13  & 0.00  & 0.00  & 25.88 & 20.08 \\
\rowcolor{gray!15}
\textbf{Flash-WAM}        & 1 & 2 & \textbf{92.30} & \underline{88.47} & \textbf{84.88} & \textbf{76.63} & \textbf{73.50} & \underline{63.25} & \textbf{88.42} & \textbf{82.66} \\
\midrule
Video-only LCM            & 1 & 1 & \underline{85.57} & \underline{78.17} & \underline{72.06} & \underline{61.81} & \underline{43.75} & \underline{34.75} & \underline{77.90} & \underline{69.46} \\
Video-only LCM + reg.     & 1 & 1 & 66.87 & 61.07 & 39.19 & 35.56 & 10.25 & 4.75  & 53.48 & 48.40 \\
Naive Joint LCM           & 1 & 1 & 54.63 & 46.00 & 21.56 & 15.63 & 0.00  & 0.00  & 39.68 & 32.96 \\
\rowcolor{gray!15}
\textbf{Flash-WAM}        & 1 & 1 & \textbf{87.30} & \textbf{86.93} & \textbf{78.44} & \textbf{72.63} & \textbf{63.50} & \textbf{60.75} & \textbf{82.56} & \textbf{80.26} \\
\bottomrule
\end{tabular}%
}
\end{table}

\paragraph{Qualitative analysis.}
Figure~\ref{fig:qualitative} shows representative frames of video predictions from an open-loop autoregressive rollout on a RoboTwin Clean-split task, in which the model generates all subsequent video chunks without intermediate observation feedback. The unaccelerated LingBot-VA teacher (25v/50a) produces clean predictions with object identity and gripper geometry preserved throughout. Both off-the-shelf distillation baselines (naive joint LCM and DMD2) degrade visibly under the same $1$v/$2$a NFE budget as our method. The brown bottle disappears under naive joint LCM, and the orange bottle becomes blurred under DMD2.
Our method preserves recognizable scene structure and object identity across the rollout. We emphasize that this figure illustrates video-prediction quality only; action precision is captured quantitatively in the success-rate tables.

\begin{wraptable}{r}{0.38\textwidth}
\vspace{0pt}
\centering
\small
\setlength{\tabcolsep}{5pt}
\caption{Flash-WAM on Motus (10-task RoboTwin 2.0 subset).}
\label{tab:motus}
\begin{tabular}{lcc}
\toprule
Method & Steps & Success (\%) \\
\midrule
Motus (teacher) & 10 & 82.8 \\
\midrule
Naive Joint LCM & 2 & 57.4 \\
\rowcolor{gray!15}
\textbf{Flash-WAM} & 2 & \textbf{81.8} \\
\midrule
Naive Joint LCM & 1 & 23.0 \\
\rowcolor{gray!15}
\textbf{Flash-WAM} & 1 & \textbf{77.0} \\
\bottomrule
\end{tabular}
\vspace{-10pt}
\end{wraptable}

\paragraph{Generality across WAMs.}
We also apply Flash-WAM to Motus~\citep{motus}, which has a different architecture. It uses separate networks for video and action and it denoises both streams together under the same step budget. Table~\ref{tab:motus} reports results on a $10$-task RoboTwin 2.0 subset with $50$ trials per task. Naive joint LCM drops from $82.8\%$ to $57.4\%$ at two steps and to $23.0\%$ at one step, whereas Flash-WAM retains $81.8\%$ and $77.0\%$. Neither the failure of naive distillation nor its recovery by Flash-WAM depends on the architecture or the noise schedule, consistent with the analysis of Section~\ref{sec:naive_fails}.

\subsection{Ablation Analysis}
\label{sec:ablation}
Table~\ref{tab:ablation_robotwin} compares Flash-WAM with three LCM-based alternatives on RoboTwin, which distill both streams uniformly (Naive joint LCM), only the video stream (Video-only LCM), or the video stream with an MSE regularizer on actions (Video-only LCM + reg.). Flash-WAM achieves the best average on both splits at both budgets and the best result in $10$ of the $12$ per-horizon columns, with the two exceptions on the Randomized split at $1$v/$2$a. Naive joint LCM collapses to $25.88\%$ on the Clean split at $1$v/$2$a, with near-zero success at horizons 2 and 3, as predicted by the analysis of Section~\ref{sec:naive_fails}. Video-only LCM avoids this collapse by leaving the action stream undistilled, but trails Flash-WAM by about $7$ points at $1$v/$2$a, showing that the action stream must also be distilled. The regularized baseline is competitive at $1$v/$2$a, $1.5$ points below Flash-WAM on the Clean split, but falls $24$ points below Video-only LCM at $1$v/$1$a, so an auxiliary loss cannot replace distillation when the action budget is tight. Flash-WAM's advantage also grows with task horizon. At $1$v/$1$a, its lead over Video-only LCM on the Clean split is $1.7$, $6.4$, and $19.8$ points on horizon-1, horizon-2, and horizon-3 tasks, consistent with action errors compounding over sequential sub-tasks.

\section{Conclusion}
\label{sec:conclusion}
We introduced Flash-WAM, a step-distillation framework for joint video-action models. Our analysis shows that off-the-shelf consistency distillation fails in this setting because its standard consistency function provides almost no learning signal at low noise levels, where precision-critical actions are refined, while video tolerates the resulting errors. Flash-WAM addresses this by deriving a separate consistency function for each stream. On LingBot-VA, Flash-WAM retains $85.5\%$ success on RoboTwin 2.0 and $96.1\%$ on LIBERO at a $19\times$ speedup, and reaches real-time latency with a single step per modality ($23\times$). The framework also transfers to Motus, whose architecture and noise schedule differ from those of LingBot-VA, retaining $77.0\%$ success at a single step compared with $82.8\%$ for the $10$-step teacher. On a Unitree G1 robot, Flash-WAM achieves $60\%$ success, compared with $43.3\%$ when only the video stream is distilled at the same step budget.

\subsection*{AI use statement}

Large language models were used in a limited role during the preparation of this paper. We used them to improve grammar, rephrase sentences for clarity, suggest alternative wordings, and tighten verbose passages, and occasionally to check that technical phrasing followed the conventions of the diffusion and step-distillation literature. All technical claims, numerical results, and conclusions reflect the authors' own findings and were not generated or substantively shaped by LLM output. The authors reviewed all AI-assisted text and take full responsibility for the content of this paper.



\bibliography{iclr2027_conference}
\bibliographystyle{iclr2027_conference}

\appendix
\section{Implementation Details}
\label{app:implementation}
We provide full implementation details for our method and all baselines. Section~\ref{app:finetuning} describes the fine-tuning phase in which the released LingBot-VA base checkpoint is adapted to each LIBERO suite prior to distillation. Section~\ref{app:flashwam_hyperparams} reports the distillation hyperparameters used to train Flash-WAM, which are largely shared between the LIBERO and RoboTwin experiments. Section~\ref{app:baselines} then specifies the implementation choices for each baseline, organized by distillation family. All training is performed on NVIDIA H100 GPUs; evaluation and latency profiling are performed on NVIDIA L40S GPUs.

\subsection{Libero Finetuning}
\label{app:finetuning}
The released LingBot-VA checkpoint is a base model trained on multi-task data~\citep{lingbotva}. Following the LingBot-VA protocol, we first fine-tune this base checkpoint separately on each LIBERO suite for $4{,}000$ training steps before applying step distillation. The fine-tuning hyperparameters are listed in Table~\ref{tab:hyperparams_finetune}. Fine-tuning on each suite will take about 24 hours on 4 H100s.
\begin{table}[H]
\centering
\caption{Hyperparameters used to fine-tune the LingBot-VA base checkpoint on each LIBERO suite.}
\label{tab:hyperparams_finetune}
\small
\begin{tabular}{ll}
\toprule
\textbf{Hyperparameter} & \textbf{Value} \\
\midrule
\multicolumn{2}{l}{\emph{Optimization}} \\
\quad Optimizer                          & AdamW, $(\beta_1, \beta_2) = (0.9, 0.95)$ \\
\quad Learning rate                      & $1 \times 10^{-5}$ \\
\quad Weight decay                       & $0.1$ \\
\quad Warmup steps                       & $10$ (linear warmup, then constant) \\
\quad Gradient clipping                  & $2.0$ \\
\midrule
\multicolumn{2}{l}{\emph{Batching}} \\
\quad Per-device batch size              & $1$ \\
\quad Gradient accumulation steps        & $30$ \\
\quad Effective batch size               & $120$ ($4 \times$ H100) \\
\midrule
\multicolumn{2}{l}{\emph{Training schedule}} \\
\quad Total training steps               & $4{,}000$ \\
\bottomrule
\end{tabular}
\end{table}

\subsection{Flash-WAM Distillation Hyperparameters}
\label{app:flashwam_hyperparams}
For LingBot-VA, video and action tokens are concatenated into the joint sequence used in pre-training and processed in a single forward pass with flex attention, so both consistency targets come from one pass of the student and one pass of the EMA target.
Starting from the fine-tuned LingBot-VA teacher, we apply Flash-WAM for $2{,}000$ steps on each LIBERO suite using the hyperparameters in Table~\ref{tab:hyperparams_libero}. Each suite takes approximately 24 hours on $4 \times$ H100 GPUs. The same training procedure is applied to all LCM-based baselines (Naive joint LCM, Video-only LCM, Video-only LCM + reg.) on RoboTwin for fair comparison.
\begin{table}[t]
\centering
\caption{Hyperparameters used to train our distilled student on LIBERO. The same configuration is used across all four LIBERO suites.}
\label{tab:hyperparams_libero}
\small
\begin{tabular}{ll}
\toprule
\textbf{Hyperparameter} & \textbf{Value} \\
\midrule
\multicolumn{2}{l}{\emph{Architecture}} \\
\quad Image resolution                & $128 \times 128$ \\
\quad Action dimension                & $30$ \\
\quad Actions per video frame         & $4$ \\
\quad Frame chunk size $K$            & $4$ \\
\midrule
\multicolumn{2}{l}{\emph{Flow matching}} \\
\quad Video SNR shift $s^v$           & $5.0$ \\
\quad Action SNR shift $s^a$          & $1.0$ \\
\midrule
\multicolumn{2}{l}{\emph{Consistency distillation}} \\
\quad Action loss weight $\lambda_a$  & $1.0$ \\
\quad EMA decay $\alpha$              & $0.995$ \\
\quad Data scale $\sigma_d$           & $0.5$ \\
\quad Loss type                       & Huber ($c = 0.001$) \\
\quad CFG range $[w_{\min}, w_{\max}]$ & $[2.0,\, 10.0]$ \\
\midrule
\multicolumn{2}{l}{\emph{Optimization}} \\
\quad Optimizer                       & AdamW, $(\beta_1, \beta_2) = (0.9, 0.999)$ \\
\quad Learning rate                   & $5 \times 10^{-6}$ \\
\quad Gradient clipping               & $2.0$ \\
\quad Warmup steps                    & $100$ \\
\quad Effective batch size            & $48$ ($4 \times$ H100) \\
\bottomrule
\end{tabular}
\end{table}

\subsection{Baseline Implementations}
\label{app:baselines}
We describe each baseline's specific implementation choices, organized by distillation family. All baselines share the training data, base checkpoint, and number of training iterations with our method (Section~\ref{app:flashwam_hyperparams}); they differ only in the distillation objective applied.

\subsubsection{Naive Joint LCM}
Naive Joint LCM applies the standard LCM consistency function~\citep{lcm} uniformly across video and action streams. The consistency function takes the form $f(\mathbf{x}_\sigma, \sigma) = c_{\text{skip}}(\sigma) \mathbf{x}_\sigma + c_{\text{out}}(\sigma) \hat{\mathbf{x}}_0$ with $c_{\text{skip}} = \sigma_d^2 / (\sigma^2 + \sigma_d^2)$ and $c_{\text{out}} = \sigma \sigma_d / \sqrt{\sigma^2 + \sigma_d^2}$. The consistency loss is computed independently for each modality and combined as $\mathcal{L} = \mathcal{L}^v + \lambda_a \mathcal{L}^a$ with $\lambda_a = 1.0$, identical to Flash-WAM's joint training objective. The only difference from Flash-WAM is that the action stream uses the same LCM parametrization as video rather than the linear-scaling parametrization.

\subsubsection{Video-only LCM}
Video-only LCM distills only the video stream. During training, the consistency loss is computed only on the video stream, and no loss is applied to the action stream. At inference, the student generates video in one step, and the undistilled action velocity field is integrated with $N^a \in \{1, 2\}$ Euler steps. All other hyperparameters match the Flash-WAM configuration in Table~\ref{tab:hyperparams_libero}.

\subsubsection{Video-only LCM + reg}
Video-only LCM + reg extends Video-only LCM by adding a flow-matching regularizer on the action stream during distillation, allowing both streams to operate at reduced NFE at inference time. The video stream is supervised by the standard LCM consistency loss as in Video-only LCM. The action stream is supervised by an MSE flow-matching loss anchored against the demonstration distribution. Specifically, given a clean ground-truth action $\mathbf{x}_0^a$ and a noise level $\sigma$ sampled from the action stream's schedule, the action input is constructed as $\mathbf{x}_\sigma^a = (1 - \sigma)\, \mathbf{x}_0^a + \sigma\, \boldsymbol{\epsilon}^a$ with $\boldsymbol{\epsilon}^a \sim \mathcal{N}(\mathbf{0}, \mathbf{I})$, and the regularizer supervises the student's action velocity prediction against the target velocity:
\begin{equation}
    v^\star = \frac{\mathbf{x}_\sigma^a - \mathbf{x}_0^a}{\sigma}.
\end{equation}
The action regularizer takes the form
\begin{equation}
    \mathcal{L}_{\text{reg}}^a = \frac{1}{|\mathcal{M}|} \big\| \mathcal{M} \odot \big(v_{\theta_S}^a - v^\star\big) \big\|^2,
\end{equation}
where $\mathcal{M}$ is the per-channel action validity mask. The full training objective combines the video consistency loss with the action regularizer:
\begin{equation}
    \mathcal{L} = \mathcal{L}^v_{\text{LCM}} + \lambda_r\, \mathcal{L}_{\text{reg}}^a.
\end{equation}
All other hyperparameters match the Flash-WAM configuration in Table~\ref{tab:hyperparams_libero}.

\subsubsection{DMD2 Baseline Implementations}
\label{app:dmd2}

DMD2~\citep{dmd2} was originally proposed for single-modality image and video distillation. Adapting it to the joint video-action diffusion regime requires several architectural and training decisions that the original method does not prescribe. We describe our adaptation choices in this section to make our two DMD2 baselines reproducible: \textbf{Video-only DMD2 + reg} (used in the main results, Table~\ref{tab:main_results}) and \textbf{Joint DMD2} (the fully-joint variant evaluated in Appendix~\ref{app:additional}).

Throughout, we use the following notation: $\theta_T$ for the unaccelerated LingBot-VA model, $\theta_S$ for the student, and $\theta_C$ for the critic. We denote the student's final output as $G_{\theta_S}(z, y) = \hat{\mathbf{x}}_0$, where $z$ is the random seed and $y$ is the conditioning context.

\paragraph{Networks.}
We maintain three full-copy networks of the joint video-action backbone:
\begin{itemize}
    \item \emph{Reference model} $\theta_T$ (frozen): the pretrained LingBot-VA model. Defines the real score.
    \item \emph{Student} $\theta_S$ (trainable): a $K$-step generator initialized from $\theta_T$.
    \item \emph{Critic} $\theta_C$ (trainable): tracks the student's joint generation distribution. Defines the fake score.
\end{itemize}
A single critic scores both modalities through the shared backbone with separate output heads (the same head structure as the reference model).

\paragraph{Variant overview.}
The two DMD2 baselines differ in two dimensions: which modalities are generated from noise during student rollouts, and which losses contribute to the student objective. Joint DMD2 (appendix variant): The student generates both video and action from pure noise via $K$-step denoising. Distribution-matching losses are applied to both modalities. There is no action regularizer.
Video-only DMD2 + reg (main paper variant): The student generates only the video stream from pure noise. The action stream input is constructed by perturbing the ground-truth action at each noise level rather than being denoised from noise. The student is supervised by distribution-matching on the video stream and an MSE-based action regularizer on the action stream.

The networks, scoring procedure, and critic objective are shared between the two variants. The variant-specific differences are flagged where they apply in the descriptions below.

\paragraph{Student rollout.}
The student denoises in $K$ steps (we use $K=4$) following a uniform noise-band schedule $1 = \sigma_0 > \sigma_1 > \cdots > \sigma_K = 0$. At each step $i$, the student computes the velocity prediction
\begin{equation}
    v_{\theta_S}\!\big(\mathbf{x}_{\sigma_i}^v, \mathbf{x}_{\sigma_i}^a, \tilde{\mathbf{x}}_0^v,\, \tilde{\mathbf{x}}_0^a,\, \sigma_i,\, y\big),
\end{equation}
from which the predicted clean output for video is recovered as $\hat{\mathbf{x}}_0^{v,(i)} = \mathbf{x}_{\sigma_i}^v - \sigma_i\, v_{\theta_S}^v$. Clean-context tokens $(\tilde{\mathbf{x}}_0^v, \tilde{\mathbf{x}}_0^a)$ supply past-frame context through flex attention, exactly as in pretraining (block-causal across chunks; strict causality from noisy to clean tokens). The video input at the next step is constructed by re-noising the predicted clean video:
\begin{equation}
    \mathbf{x}_{\sigma_{i+1}}^v = (1-\sigma_{i+1})\, \hat{\mathbf{x}}_0^{v,(i)} + \sigma_{i+1}\, \boldsymbol{\epsilon}^{v,(i+1)}, \qquad \boldsymbol{\epsilon}^{v,(i+1)} \sim \mathcal{N}(\mathbf{0}, \mathbf{I}).
\end{equation}

The action stream input depends on the variant. For Joint DMD2, the action stream is denoised symmetrically with the video stream:
\begin{equation}
    \mathbf{x}_{\sigma_{i+1}}^a = (1-\sigma_{i+1})\, \hat{\mathbf{x}}_0^{a,(i)} + \sigma_{i+1}\, \boldsymbol{\epsilon}^{a,(i+1)},
\end{equation}
where $\hat{\mathbf{x}}_0^{a,(i)} = \mathbf{x}_{\sigma_i}^a - \sigma_i\, v_{\theta_S}^a$. For Video-only DMD2 + reg, the action stream input at each step is constructed by directly perturbing the ground-truth action at the corresponding noise level:
\begin{equation}
    \mathbf{x}_{\sigma_i}^a = (1 - \sigma_i)\, \mathbf{x}_0^a + \sigma_i\, \boldsymbol{\epsilon}^{a,i}, \qquad \boldsymbol{\epsilon}^{a,i} \sim \mathcal{N}(\mathbf{0}, \mathbf{I}),
\end{equation}
with $\mathbf{x}_0^a$ the ground-truth action. To bound activation memory, only the final step ($i = K-1$) retains autograd; steps $0, \ldots, K-2$ run under \texttt{no\_grad}.

\paragraph{Single-pass joint scoring.}
Given a student rollout output, both modalities are re-noised at independently sampled noise levels $\sigma^v, \sigma^a \sim \mathcal{U}(0.02, 0.98)$ and presented jointly to both the critic and the reference model:
\begin{equation}
    \tilde{\mathbf{x}}^v = (1 - \sigma^v)\, \hat{\mathbf{x}}_0^v + \sigma^v\, \boldsymbol{\eta}^v.
\end{equation}
For Joint DMD2, the action input to scoring is constructed analogously by re-noising the student's predicted action $\hat{\mathbf{x}}_0^a$. For Video-only DMD2 + reg, the action input to scoring is constructed by perturbing the ground-truth action $\mathbf{x}_0^a$ at noise level $\sigma^a$ directly. In both variants, the critic and reference model produce their predictions in a single forward pass on the joint input $(\tilde{\mathbf{x}}^v, \tilde{\mathbf{x}}^a)$.

The fake and real predicted clean outputs for video are
\begin{equation}
    \hat{\mathbf{x}}_0^{v,\text{fake}} = \tilde{\mathbf{x}}^v - \sigma^v\, v_{\theta_C}^v\!\big(\tilde{\mathbf{x}}^v, \tilde{\mathbf{x}}^a, \sigma^v, \sigma^a, y\big),
\end{equation}
\begin{equation}
    \hat{\mathbf{x}}_0^{v,\text{real}} = \tilde{\mathbf{x}}^v - \sigma^v\, v_{\theta_T}^{v,\text{cfg}}\!\big(\tilde{\mathbf{x}}^v, \tilde{\mathbf{x}}^a, \sigma^v, \sigma^a, y\big),
\end{equation}
with classifier-free guidance applied only to the video real score:
\begin{equation}
    v_{\theta_T}^{v,\text{cfg}} = v_{\theta_T}^v(\cdot \mid \varnothing) + w^v\, \big[v_{\theta_T}^v(\cdot \mid y) - v_{\theta_T}^v(\cdot \mid \varnothing)\big], \qquad w^v = 3.0.
\end{equation}
For Joint DMD2, the analogous fake and real action outputs are computed from the critic and reference model's action heads (without CFG): $\hat{\mathbf{x}}_0^{a,\text{real}} = \tilde{\mathbf{x}}^a - \sigma^a\, v_{\theta_T}^a(\cdot \mid y)$. Total scoring cost per joint sample is three forward passes (one critic, two reference: conditioned and unconditioned).

\paragraph{Distribution-matching losses.}
The DMD2 distribution-matching gradient pushes the student's output distribution toward the reference distribution. For video,
\begin{equation}
    \mathcal{L}_{\text{DM}}^v = \tfrac{1}{2} \big\| G_{\theta_S}^v(z, y) - \mathrm{sg}\!\big[G_{\theta_S}^v(z, y) - g^v\big] \big\|^2,
\end{equation}
where the gradient surrogate is
\begin{equation}
    g^v = \frac{\hat{\mathbf{x}}_0^{v,\text{fake}} - \hat{\mathbf{x}}_0^{v,\text{real}}}{\big\| \hat{\mathbf{x}}_0^v - \hat{\mathbf{x}}_0^{v,\text{real}} \big\|_1 + \varepsilon},
\end{equation}
with $\varepsilon = 10^{-8}$ and per-sample $L_1$ normalization following the DMD2 gradient-norm fix. For Joint DMD2, an analogous distribution-matching loss is applied to the action stream:
\begin{equation}
    g^a = \frac{\hat{\mathbf{x}}_0^{a,\text{fake}} - \hat{\mathbf{x}}_0^{a,\text{real}}}{\big\| \hat{\mathbf{x}}_0^a - \hat{\mathbf{x}}_0^{a,\text{real}} \big\|_1 + \varepsilon},
\end{equation}
\begin{equation}
    \mathcal{L}_{\text{DM}}^a = \frac{1}{|\mathcal{M}|} \big\| \mathcal{M} \odot \big(G_{\theta_S}^a - \mathrm{sg}\!\big[G_{\theta_S}^a - g^a\big]\big) \big\|^2,
\end{equation}
where $\mathcal{M}$ is the per-channel action validity mask. For Video-only DMD2 + reg, no distribution-matching loss is applied to the action stream.

\paragraph{Action regularizer (Video-only DMD2 + reg only).}
For the Video-only DMD2 + reg variant, we anchor the student's action head to the demonstration distribution with a flow-matching loss evaluated on the action input at the student's final generation step. Let $\mathbf{x}_{\sigma_{K-1}}^a$ denote the action input at step $K-1$ (the perturbed ground-truth action). The target velocity that maps this input to the clean ground-truth action is
\begin{equation}
    v^\star = \frac{\mathbf{x}_{\sigma_{K-1}}^a - \mathbf{x}_0^a}{\sigma_{K-1}}.
\end{equation}
We supervise the student's final-step action prediction against $v^\star$:
\begin{equation}
    \mathcal{L}_{\text{reg}}^a = \frac{1}{|\mathcal{M}|} \big\| \mathcal{M} \odot \big(v_{\theta_S}^a - v^\star\big) \big\|^2.
\end{equation}

\paragraph{Critic objective.}
The critic is trained to denoise the student's joint distribution via flow matching. For each rollout, fresh noise levels $\sigma^{v\prime}, \sigma^{a\prime}$ and noise $\boldsymbol{\eta}^{v\prime}, \boldsymbol{\eta}^{a\prime}$ are drawn, the student's joint output is re-noised, and the critic minimizes
\begin{equation}
    \mathcal{L}_{\text{critic}} = w(\sigma^{v\prime}) \big\| v_{\theta_C}^v - (\boldsymbol{\eta}^{v\prime} - \hat{\mathbf{x}}_0^v) \big\|^2 + w(\sigma^{a\prime}) \big\| \mathcal{M} \odot \big[v_{\theta_C}^a - (\boldsymbol{\eta}^{a\prime} - \hat{\mathbf{x}}_0^a)\big] \big\|^2,
\end{equation}
where $w(\sigma) = \exp\!\big(-2((\sigma - 0.5)/T)^2\big)$ is a bell-curve timestep weight inherited from pretraining. The critic and student updates share a single forward pass through the joint input.

The two DMD2 baselines combine the loss terms above as follows.

Joint DMD2:
\begin{equation}
    \mathcal{L}_{\theta_S}^{\text{joint}} = \lambda_{\text{DM}}^v\, \mathcal{L}_{\text{DM}}^v + \lambda_{\text{DM}}^a\, \mathcal{L}_{\text{DM}}^a.
\end{equation}

Video-only DMD2 + reg:
\begin{equation}
    \mathcal{L}_{\theta_S}^{\text{V-only}} = \lambda_{\text{DM}}^v\, \mathcal{L}_{\text{DM}}^v + \lambda_{\text{reg}}\, \mathcal{L}_{\text{reg}}^a.
\end{equation}

We use $\lambda_{\text{DM}}^v = 1.0$, $\lambda_{\text{DM}}^a = 0.1$, $\lambda_{\text{reg}} = 1.0$. Following DMD2's update schedule, the critic updates every iteration while the student updates every $T_g = 5$ iterations.

\paragraph{Hyperparameters.}
We use AdamW with $\beta_1 = 0.9$, $\beta_2 = 0.999$, $\varepsilon = 10^{-8}$, no weight decay, gradient clipping $\|g\|_2 \le 2.0$, and a $100$-step linear warmup to constant learning rate. The student learning rate is $5 \times 10^{-7}$ and the critic learning rate is $10^{-6}$. Both DMD2 variants are trained on $4 \times$ NVIDIA H100 GPUs for $2{,}000$ steps, matching the LCM-based training schedule.

\section{Additional Experimental Results}
\label{app:additional}

\begin{table}[t]
\centering
\caption{All distillation strategies on RoboTwin 2.0 at the single-step budget ($1$v/$1$a), including both DMD2 variants, broken down by task horizon and averaged across all 50 tasks.}
\label{tab:ablation_robotwin_app}
\resizebox{\textwidth}{!}{%
\begin{tabular}{lcccccccccc}
\toprule
& & & \multicolumn{2}{c}{Horizon = 1} & \multicolumn{2}{c}{Horizon = 2} & \multicolumn{2}{c}{Horizon = 3} & \multicolumn{2}{c}{Average (50 tasks)} \\
\cmidrule(lr){4-5}\cmidrule(lr){6-7}\cmidrule(lr){8-9}\cmidrule(lr){10-11}
Method & $N^v$ & $N^a$ & Clean & Rand. & Clean & Rand. & Clean & Rand. & Clean & Rand. \\
\midrule
Joint DMD2           & 1 & 1 & 67.50 & 62.13 & 36.50 & 33.81 & 6.00 & 4.50 & 52.66 & 48.46 \\
DMD2 + reg.           & 1 & 1 & 76.93 & 72.23 & 54.81 & 50.25 & 25.00 & 11.25 & 66.53 & 60.32 \\
Video-only LCM            & 1 & 1 & \underline{85.57} & \underline{78.17} & \underline{72.06} & \underline{61.81} & \underline{43.75} & \underline{34.75} & \underline{77.90} & \underline{69.46} \\
Video-only LCM + reg.     & 1 & 1 & 66.87 & 61.07 & 39.19 & 35.56 & 10.25 & 4.75  & 53.48 & 48.40 \\
Naive joint LCM           & 1 & 1 & 54.63    & 46.00    & 21.56    & 15.63    & 0.00   & 0.00    & 39.68    & 32.96    \\
\rowcolor{gray!15}
\textbf{Flash-WAM}             & 1 & 1 & \textbf{87.30} & \textbf{86.93} & \textbf{78.44} & \textbf{72.63} & \textbf{63.50} & \textbf{60.75} & \textbf{82.56} & \textbf{80.26} \\
\bottomrule
\end{tabular}%
}
\end{table}

Table~\ref{tab:ablation_robotwin_app} compares all distillation strategies at the single-step budget ($1$v/$1$a) on RoboTwin 2.0, including Joint DMD2 (Appendix~\ref{app:dmd2}), which distills both streams with distribution matching. Joint DMD2 reaches $52.7\%$ Clean and $48.5\%$ Randomized success on average, roughly $14$ points below Video-only DMD2 + reg, the variant reported in Table~\ref{tab:main_results}, and $30$ points below Flash-WAM. The same pattern appears for LCM in Table~\ref{tab:ablation_robotwin}. Applying one distillation objective uniformly to both streams fails for both distillation families, and restricting it to the video stream helps but still trails Flash-WAM by a wide margin.

\begin{table}
\centering
\caption{Per-task success rate results on Robotwin 2.0.}
\label{tab:robotwin_full}
\resizebox{\textwidth}{!}{%
\begin{tabular}{lc cc cc cc cc}
\toprule
\multirow{2}{*}{Simulation Task} & \multirow{2}{*}{Horizon} & \multicolumn{2}{c}{Flash-WAM (1v/2a)} & \multicolumn{2}{c}{Flash-WAM(1v/1a)} & \multicolumn{2}{c}{Naive Joint LCM(1v/1a)} & \multicolumn{2}{c}{DMD2 (1v/1a)} \\
\cmidrule(lr){3-4} \cmidrule(lr){5-6} \cmidrule(lr){7-8} \cmidrule(lr){9-10}
& & Clean & Rand. & Clean & Rand. & Clean & Rand. & Clean & Rand. \\
\midrule
Adjust Bottle              & 1 &  98 &  98 &  98 &  98 &  93 &  88 &  98 &  85  \\
Beat Block Hammer          & 1 &  98 &  95 &  99 &  97 &  61 &  20 &  77 &  74  \\
Blocks Ranking RGB         & 3 &  86 &  85 &  73 &  74 &   0 &   0 &   8 &   5  \\
Blocks Ranking Size        & 3 &  68 &  46 &  72 &  65 &   0 &   0 &   5 &   1  \\
Click Alarmclock           & 1 & 100 & 100 & 100 & 100 &  81 &  86 & 100 & 100  \\
Click Bell                 & 1 & 100 & 100 & 100 & 100 & 100 & 100 & 100 &  99  \\
Dump Bin Bigbin            & 1 &  95 &  98 &  94 &  93 &  71 &  61 &  67 &  57  \\
Grab Roller                & 1 & 100 & 100 & 100 & 100 &  98 &  94 &  97 &  97  \\
Handover Block             & 2 &  91 &  51 &  71 &  34 &   0 &   0 &  22 &  11  \\
Handover Mic               & 2 &  69 &  71 &  66 &  63 &  44 &  25 &  11 &  22  \\
Hanging Mug                & 2 &  36 &  39 &  32 &  28 &   5 &   2 &   3 &   4  \\
Lift Pot                   & 1 &  99 & 100 &  98 &  95 &   2 &   8 &  25 &  11  \\
Move Can Pot               & 1 &  91 &  80 &  95 &  92 &   2 &   3 &  33 &  16  \\
Move Pillbottle Pad        & 1 &  99 &  94 &  93 &  89 &  25 &  17 &  58 &  58  \\
Move Playingcard Away      & 1 & 100 &  96 & 100 & 100 &  86 &  82 &  97 &  90  \\
Move Stapler Pad           & 1 &  61 &  46 &  39 &  36 &  14 &  15 &  16 &  10  \\
Open Laptop                & 1 &  94 &  87 &  95 &  94 &  74 &  70 &  64 &  62  \\
Open Microwave             & 1 &  67 &  70 &  25 &  27 &  26 &  27 &  59 &  62  \\
Pick Diverse Bottles       & 2 &  92 &  66 &  92 &  88 &  11 &   4 &  60 &  42  \\
Pick Dual Bottles          & 2 & 100 &  85 &  99 &  84 &   4 &   3 &  74 &  56  \\
Place A2B Left             & 1 &  91 &  81 &  88 &  93 &  63 &  49 &  72 &  57  \\
Place A2B Right            & 1 &  92 &  91 &  91 &  92 &  61 &  46 &  68 &  67  \\
Place Bread Basket         & 1 &  93 &  85 &  89 &  72 &  51 &  31 &  55 &  46  \\
Place Bread Skillet        & 2 &  89 &  85 &  86 &  88 &  53 &  44 &  51 &  46  \\
Place Burger Fries         & 2 &  97 &  95 &  94 &  93 &  74 &  63 &  80 &  81  \\
Place Can Basket           & 2 &  83 &  75 &  79 &  76 &  16 &  15 &  18 &  25  \\
Place Cans Plasticbox      & 2 & 100 &  97 &  96 &  97 &   2 &   1 &   5 &   4  \\
Place Container Plate      & 1 &  99 &  98 &  97 &  97 &  87 &  67 &  88 &  94  \\
Place Dual Shoes           & 2 &  78 &  81 &  65 &  64 &   0 &   2 &  20 &  21  \\
Place Empty Cup            & 1 & 100 &  98 &  99 &  99 &  26 &  27 &  69 &  67  \\
Place Fan                  & 1 &  83 &  77 &  65 &  78 &  25 &  20 &  21 &  34  \\
Place Mouse Pad            & 1 &  89 &  84 &  85 &  80 &  33 &  18 &  48 &  46  \\
Place Object Basket        & 2 &  87 &  77 &  78 &  84 &  17 &  14 &  59 &  47  \\
Place Object Scale         & 1 &  95 &  86 &  96 &  97 &  47 &  26 &  66 &  52  \\
Place Object Stand         & 1 & 100 &  95 &  95 &  92 &  15 &   9 &  69 &  72  \\
Place Phone Stand          & 1 &  97 &  94 &  96 &  95 &  60 &  45 &  87 &  74  \\
Place Shoe                 & 1 &  91 &  92 &  56 &  79 &  32 &  23 &  55 &  46  \\
Press Stapler              & 1 &  93 &  93 &  89 &  93 &  88 &  78 &  84 &  78  \\
Put Bottles Dustbin        & 3 &  44 &  30 &  19 &  15 &   0 &   0 &   0 &   0  \\
Put Object Cabinet         & 2 &  79 &  54 &  62 &  39 &   3 &   1 &   6 &   2  \\
Rotate QRcode              & 1 &  94 &  92 &  88 &  82 &  51 &  60 &  56 &  48  \\
Scan Object                & 2 &  88 &  80 &  78 &  68 &  25 &  22 &  29 &  22  \\
Shake Bottle Horizontally  & 1 &  99 &  95 &  99 &  95 &  97 &  88 &  99 &  94  \\
Shake Bottle               & 1 &  99 &  97 & 100 &  97 &  94 &  85 &  99 &  95  \\
Stack Blocks Three         & 3 &  96 &  92 &  90 &  89 &   0 &   0 &  11 &  12  \\
Stack Blocks Two           & 2 & 100 &  99 & 100 &  99 &  10 &   4 &  67 &  57  \\
Stack Bowls Three          & 3 &  76 &  78 &  60 &  67 &  24 &  10 &  19 &  35  \\
Stack Bowls Two            & 2 &  93 &  93 &  97 &  90 &  57 &  40 &  60 &  66  \\
Stamp Seal                 & 1 &  96 &  83 &  89 &  85 &  14 &   6 &  31 &  30  \\
Turn Switch                & 1 &  56 &  49 &  61 &  61 &  62 &  49 &  67 &  43  \\
\midrule
\textbf{Average (\%)}      & -- & 88.42 & 82.66 & 82.56 & 80.26 & 39.68 & 32.96 & 52.66 & 48.46 \\
\bottomrule
\end{tabular}%
}
\end{table}

\end{document}